\documentclass[11pt]{article}
\usepackage{amsmath,amssymb,amsfonts} 
\usepackage{epsfig} \usepackage{latexsym,nicefrac,bbm}
\usepackage{xspace}
\usepackage{color,fancybox,graphicx,fullpage}
\usepackage[top=1in, bottom=1in, left=1in, right=1in]{geometry}
\usepackage{tabularx} 
\usepackage{hyperref} 
\usepackage{pdfsync}
\usepackage{multicol}
\usepackage{physics}

\usepackage{enumitem}
\usepackage{algorithm}
\usepackage{algpseudocode}
\usepackage{tikz}
\usetikzlibrary{calc}
\usepackage{amssymb}
\usepackage{amsmath}
\usepackage{amsthm}
\usepackage{multirow}
\usepackage{booktabs} % for professional tables
\usepackage[justification=centering]{subfig}

\usepackage{array}
\usepackage{arydshln}
\def\qed{\hfill\ensuremath{\square}}

\newcommand\R{\mathbb R}

\newtheorem{theorem}{Theorem}[section]

\newtheorem{definition}{Definition}[section]

\newtheorem{remark}[theorem]{Remark}

\renewcommand{\norm}[1]{\ensuremath{\left\lVert #1 \right\rVert}}
\renewcommand{\norm}[1]{\lVert #1 \rVert}
\newcommand{\br}[1]{\left\{#1\right\}}           
\newcommand{\inp}[1]{\langle #1\rangle}
\def\h_#1{\hat{#1}}
\def\wh_#1{\widehat{#1}}

\title{\bf Improved Adversarial Learning for Fair Classification}

\usepackage{authblk}

\author[1]{L. Elisa Celis}
\author[2]{Vijay Keswani}
\affil[1]{Yale University}
\affil[2]{\small \'{E}cole Polytechnique F\'{e}d\'{e}rale de Lausanne (EPFL), Switzerland}

\begin{document}

\maketitle
 
\begin{abstract}
Motivated by concerns that machine learning algorithms may introduce significant bias in classification models, developing fair classifiers has become an important problem in machine learning research.
One important paradigm towards this has been providing algorithms for adversarially learning fair classifiers \cite{madras2018learning, zhang2018mitigating}.
We formulate the adversarial learning problem as a multi-objective optimization problem and find the fair model using gradient descent-ascent algorithm with a
modified gradient update step, inspired by the approach of \cite{zhang2018mitigating}.
We provide theoretical insight and guarantees that formalize the heuristic arguments presented previously towards taking such an approach.
We test our approach empirically on the Adult dataset and synthetic datasets and compare against state of the art algorithms \cite{celis2018classification, zafar2017fairness, zhang2018mitigating}.
The results show that our models and algorithms have comparable or better accuracy than other algorithms while performing better in terms of fairness, as measured using statistical rate or false discovery rate.

\end{abstract}

\newpage

\section{Introduction}
Recent studies have brought to attention the problem of bias in machine learning algorithms. 
\cite{kay2015unequal} showed that Google Image search results for occupations are more gender-biased compared to the ground truth and that this affects the perspective of women in these occupations.
\cite{flores2016false, berk2009role} argued that classification algorithm to predict criminal recidivism can be race-biased due to biased data and algorithms.
Even in the context of online advertising, \cite{datta2015automated} demonstrated that women are less likely to be shown advertisements of high-paying jobs. 
With such an important gap in the design of these algorithms, it is important to study and provide models for fair classifiers for all kinds of datasets and settings.

We consider the paradigm of adversarial learning to design fair classifiers.
The goal of adversarial learning is to pit the training algorithm against an adversary which tries to determine whether the trained model is ``robust'' enough or not. 
Popularized by Generative Adversarial Networks (GANs) \cite{goodfellow2014generative}, which are used generate fake samples from the unknown training distribution, the adversary's job is to decide whether the sample generated is real or fake and give its feedback to the generator, which then uses the feedback to improve the model.
In applications such as fair classification, an adversary can be introduced to check whether the trained classifier is fair or not.
If the model is deemed to be not fair in terms of the chosen metric, the training algorithm uses the feedback from the adversary to modify the model and the process is repeated.

%
%\vspace{-0.1in}
\subsection{Our Contributions: }
While adversarial fairness has been explored in other papers such as \cite{zhang2018mitigating, wadsworth2018achieving}, we theoretically and experimentally analyze the current suggested models and suggest better performing algorithms and models.
We employ a fairness-metric specific model for the adversary (Section~\ref{sec:model}) and show that it performs better than \cite{zhang2018mitigating} and other related work for real-world and adversarial datasets (Section~\ref{sec:experiments}).

The \textit{modified gradient} update algorithm we use is similar to the work of \cite{zhang2018mitigating}, but we suggest certain variations to improve the performance, for example, we employ the usa of Accelerated Gradient Descent for noisy gradient oracles \cite{cohen2018acceleration}, which results in a more efficient implementation (Section~\ref{sec:algorithms}).
We also discuss, theoretically and empirically, the difference between the normal and modified gradient update and correspondingly motivate the use of the modified update (Section~\ref{sec:analysis_without_proj}).

We present a general theoretical analysis for the convergence of the Normal Gradient Descent and the Accelerated Gradient Descent with modified gradient updates for multi-objective optimization and give a quantification of the \textit{price of fairness} incurred when perfect fairness is to be ensured (Section~\ref{sec:mod_theory}).
Finally, we design and implement a model to ensure false discovery parity and show that the adversarial model can be extended to other fairness metrics; this is presented in Section~\ref{sec:fdr}.

\subsection{Notation}
Let $S = (x_i, y_i, z_i)_{i \in [N]}$ \footnote{$[N] = \br{1, \dots, N}$.} be the training set of $N$ samples, where $x_i \in \R^n$ is the feature vector of the $i$-th element, $z_i$ is the sensitive attribute of the $i$-th element and $y_i$ is it's class label. 
Let $G_j := \br{i \in [N] \mid z_i = j}$ denote the samples with sensitive attribute value $j$.
We assume that the sensitive attribute and the class label are binary, i.e., $y_i, z_i \in \br{0,1}$ for all $i \in [N]$.

The goal is to design a classifier, denoted by $f : \R^n \rightarrow \br{0, 1}$. Let $L_C$ denote the classification loss (exact expression to be specified later) and $L_F$ denote the fairness adversary loss. 
We will often use logistic regression as the classifier with $\textrm{log-loss}$ as the cost function. For a classifier $f : \R^n \rightarrow [0,1]$, it is defined as the following.
\begin{align*}
\textrm{log-loss}_S(f) = - \frac{1}{|S|} \sum_{x_i, y_i \in S} ( y_i \log(f(x_i)) + (1 - y_i) \log(1 - f(x_i))). 
\end{align*}
We will also use $\sigma(\cdot)$ to denote the sigmoid function, i.e., 
$\sigma(z) = (1 + e^{-z})^{-1},$
and $\Pi_v u$ will be used to denote the projection of vector $u$ on $v$, i.e.,
$$\Pi_v u := \frac{\inp{u,v}}{\inp{v,v}} \cdot v.$$
For theoretical analysis, we may need to assume certain properties on the loss function; in particular, the smoothness property.
A function $f$ is $L$-Lipschitz smooth if for any $u$ and $v$,
$$\norm{\nabla f(u) - \nabla f(u)} \leq L\norm{u-v}. $$
To compare the correlation between vectors, we will use the Pearson correlation coefficient \cite{PearsonCorrelation} which, for two vectors $u, v \in \R^n$, is defined as
\[\frac{\sum_{i =1}^n (u_i - \bar{u})(v_i - \bar{v})}{\sqrt{\sum_{i =1}^n (u_i - \bar{u})^2} \sqrt{\sum_{i =1}^n (v_i - \bar{v})^2}},\]
where $\bar{u}, \bar{v}$ are the means of vectors $u,v$ respectively.

\subsection{Fairness Metrics}
The fairness goal to be satisfied will depend on the fairness metric used. 
We will work with two fairness metrics in this document, statistical parity and false discovery rate, but with appropriate changes, the algorithm and the model can be used for other metrics as well.

\begin{definition}[Statistical Parity]
Given a dataset $S = (x_i, y_i, z_i)_{i \in [N]}$, a classifier $f$ satisfies statistical parity if
$\mathbb{P}[f = 1 \mid G_1] = \mathbb{P}[f = 1 \mid G_0],$
i.e., the probability of positive classification is same for all values of sensitive attribute.
\end{definition}
\noindent
Statistical parity has also been called demographic parity or disparate impact in many related works \cite{zafar2017fairness}. The above condition of parity can also be relaxed to ensure better fairness in certain cases. 
Correspondingly, we also design our algorithm to take the statistical rate $\tau$ as input. It is defined as follows.
\begin{definition}[Statistical Rate]\label{def:sr}
Given a dataset $S = (x_i, y_i, z_i)_{i \in [N]}$ and $0 < \tau \leq 1$, a classifier $f$ has statistical rate $\tau$ if
\[ \tau =  \min \left\{ \frac{\mathbb{P}[f = 1 \mid G_1]}{\mathbb{P}[f = 1 \mid G_0]}, \frac{\mathbb{P}[f = 1 \mid G_0]}{\mathbb{P}[f = 1 \mid G_1]} \right\}.\]
\end{definition}
\noindent
Similarly, false discovery parity is satisfied when the probability of error given positive classification is equal for all sensitive attribute values. Formally, false discovery rate is defined as the following.
\begin{definition}[False Discovery Rate] \label{def:fdr}
Given a dataset $S = (x_i, y_i, z_i)_{i \in [N]}$ and $0 < \tau \leq 1$, a classifier $f$ has false discovery rate $\tau$ if
\[ \tau =  \min \left\{ \frac{\mathbb{P}[Y = 0 \mid f = 1,  G_1]}{\mathbb{P}[Y = 0 \mid f = 1, G_0]}, \frac{\mathbb{P}[Y = 0 \mid f = 1, G_0]}{\mathbb{P}[Y = 0 \mid f = 1, G_1]} \right\}.\]
\end{definition}
\noindent
There are many other fairness metrics that have been considered before, for example, false positive rate, false negative rate, true positive rate, true negative rate, false omission rate, equalized odds, etc. 
The goal of this paper is not to provide a meta-classifier, but rather to understand how to ensure fairness using adversarial programs. 
However, using the framework described in Section~\ref{sec:model}, one can design adversarial programs to ensure different fairness parities, as we demonstrate with statistical rate and false discovery rate.
The reason for choosing false discovery rate over other similar metrics like false negative rate is that the probability is conditioned over the classification label.
Such metrics are useful in cases when false prediction results in additional costs, for example, when classifying whether a person has a medical condition or not. 
False discovery rate has been considered in relatively few earlier works \cite{celis2018classification} and we show that our framework can be used to ensure false discovery parity.

\section{Related Work}
The idea of adversarial machine learning was popularized by the introduction of Generative Adversarial Networks (GANs) \cite{goodfellow2014generative}.
Based on similar ideas, multiple learning algorithms have been suggested to generate fair classifiers using adversaries.

As mentioned earlier, \cite{zhang2018mitigating} proposed a model to learn a fair classifier based on the idea of adversarial debiasing. 
Their algorithm also uses the gradient descent with the modified update but does not include any theoretical guarantees on convergence, and experimentally, they suggest a model to just ensure equalized odds for the Adult dataset.
\cite{wadsworth2018achieving} use a similar adversarial model for COMPAS dataset \cite{compas}, but do not use the modified update for optimization.

A few other papers use adversarial settings to tackle the problem in a different manner.
\cite{madras2018learning} learn a ``fair'' latent representation of the input data and then design a classifier using the representation.
The adversary's job here is to ensure that the representation generated is fair.
On the other hand, \cite{xu2018fairgan} use the GAN framework to generate fair synthetic data from the original training data, which can then be used to train a classifier.
To do so, they add another discriminator to the network which checks for fairness.
The work of \cite{xu2018fairgan} is more comparable to the pre-processing algorithm given by \cite{kamiran2012data, feldman2015certifying, krasanakis2018adaptive, calmon2017optimized}.

Many other models for fair classification have been proposed. They differ either in the formulation of the problem or the fairness metric considered. We try to summarize the major results below.

While \cite{kamiran2012data, feldman2015certifying, krasanakis2018adaptive, calmon2017optimized} gave pre-processing algorithms to ensure that the data is fair, most other fair classification algorithms suggest a new optimization problem or use post-processing techniques.
\cite{celis2018classification, menon2018cost, corbett2017algorithmic} formulate the problem as Bayesian classification problem with fairness constraints and suggest methods to reduce it to an unconstrained optimization problem.
\cite{zafar2017fairness,zafar2017fairnessmis} suggest a covariance-based constraint which they argue can be used to ensure statistical parity or equalized odds.
To deal with the issue of multiple fairness metrics, \cite{celis2018classification, agarwal2018reductions, quadrianto2017recycling} give a unified in-processing framework to ensure fairness with respect to different metrics.
Most of these algorithms can be seen as formulating a regularizer function to ensure fairness and as discussed earlier, such an algorithm may not be able to ensure fairness when the dataset is adversarial.
Also, to the best of our knowledge, only \cite{celis2018classification, quadrianto2017recycling} provide classifiers that can ensure false discovery parity.

The work of \cite{agarwal2018reductions} is perhaps the closest in terms of the techniques involved. 
They formulate their constrained optimization problem as an unconstrained one using Lagrangian transformation. 
This leads to a min-max optimization problem, which they then solve using the saddle point methods of \cite{freund1996game, kivinen1997exponentiated}.
The key difference with respect to our work is that they do not aim to learn the sensitive attribute information from the classifier and instead just use the regularizer.
% and that they do not analyze the algorithm for adversarially constructed datasets.
Furthermore, their formulation does not support \textit{non-linear} metrics like false discovery rate.

To ensure fairness by post-processing, \cite{hardt2016equality} gave a simple algorithm to find the appropriate threshold for a trained classifier to ensure equalized odds. 
Similarly, \cite{goh2016satisfying, pleiss2017fairness, woodworth2017learning} suggest different ways of fixing a decision boundary for different values of the sensitive attribute to ensure that the final classifier is fair.

\begin{figure*}[t]
\centering
  \includegraphics[width=\linewidth]{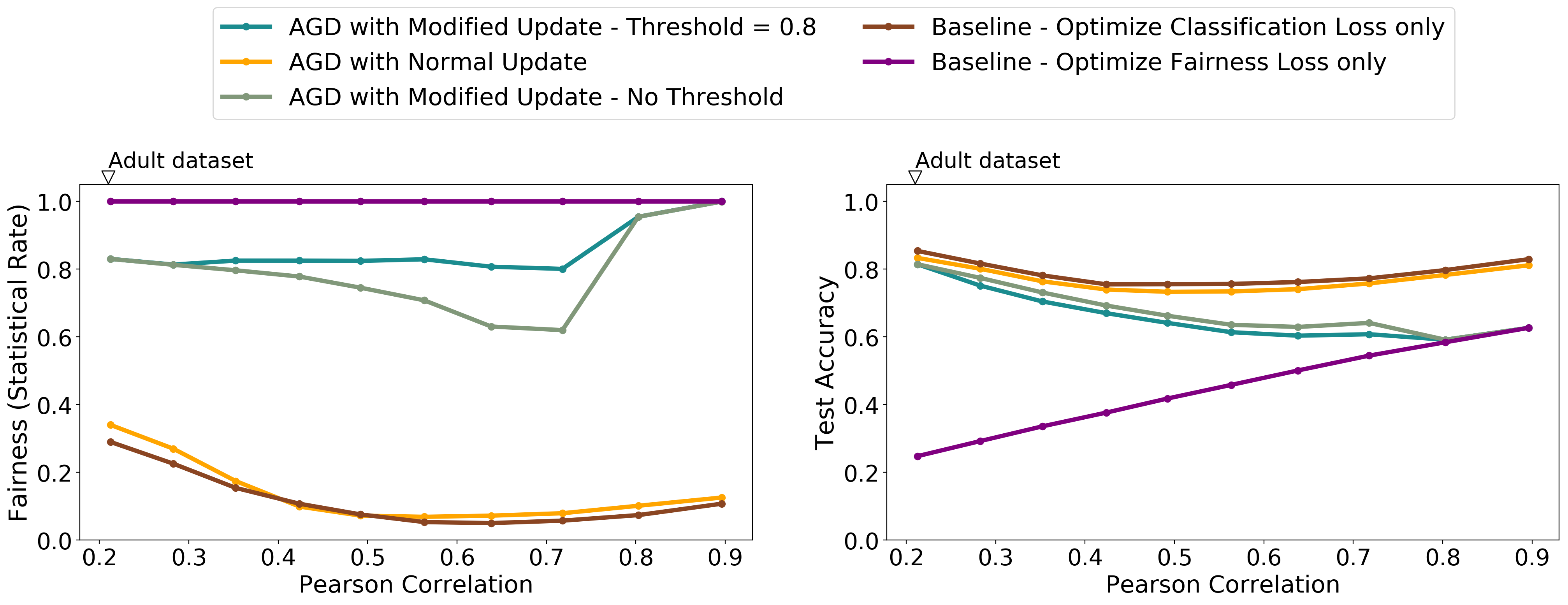}
  \subfloat[Correlation vs Fairness.]{\hspace{.5\linewidth}}
\subfloat[Correlation vs Accuracy.]{\hspace{.5\linewidth}}
\caption{ The figures plot accuracy and statistical rate for different synthetic datasets with varying correlation between the class label and the sensitive attribute.
The algorithms used are AGD with modified update without threshold, AGD with modified update with threshold 0.8, AGD with normal update and baseline algorithms where we optimize only accuracy or only fairness. 
The statistical rate achieved by Modified update is always higher than Normal Gradient Update.
Note that the left-most point on $x$-axis, with correlation = 0.21, represents the Adult dataset (pointed by the arrow).
}
\label{fig:figure9}
\end{figure*}

\begin{figure*}[t]
\centering
  \includegraphics[width=\linewidth]{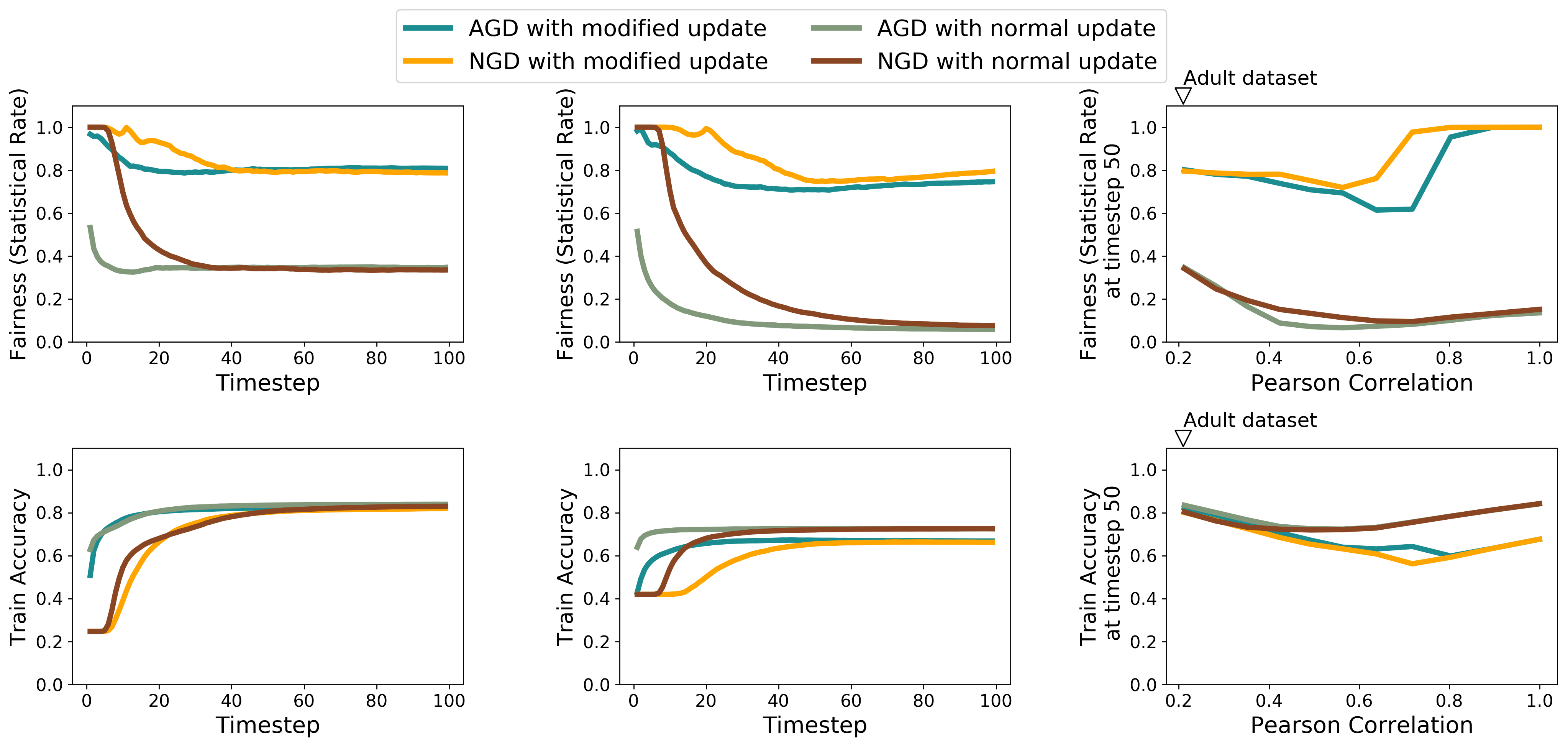}
  \subfloat[\scriptsize{Adult Dataset, Correlation = 0.21.}\label{fig:compare_agd_ngd_adult}]{\hspace{.33\linewidth}}
\subfloat[\scriptsize{Synthetic Dataset, Correlation = 0.50.}\label{fig:compare_agd_ngd_syn}]{\hspace{.33\linewidth}}
\subfloat[\scriptsize{Correlation vs Training Accuracy at Timestep 50.}\label{fig:compare_corr_train}]{\hspace{.33\linewidth}}
\caption{ Plots in (a) represent the Training accuracy and Statistical rate across the Gradient Descent for different algorithms on the Adult dataset and plots in (b) represents the same variables for a synthetic dataset with correlation = 0.5.
Note that the algorithms with modified update perform achieve much better fairness than algorithms with normal update and AGD with modified update converges to good accuracy and fairness faster than NGD with modified update.
Plots in (c) show how the training accuracy and fairness at timestep 50 varies for datasets with different correlation coefficients.
Once again, the algorithms with modified update perform better. The first point on the x-axis represents the Adult dataset, with correlation = 0.21.
AGD is Accelerated Gradient Descent and NGD is Normal Gradient Descent.}
\label{fig:figure1}
\end{figure*}

\section{Model} \label{sec:model}
Let $L_C$ denote the classification loss and let $L_F$ denote the adversary loss. 
The adversary, given parameters $w$ of the classifier, uses extra parameters, say $u \in \R^m$, to deduce the sensitive attribute information from the classifier. 
The job is to find the classifier $f$ which minimizes $L_C$, while the adversary tries to maximize $L_F$. Formally, we want to find the parameters $w^*, u^*$ such that
\[L_C(w^*) = \min_{w} L_C(w),\]
\[L_F(u^*, w^*) = \max_{w,u} L_F(w,u).\]
In practice, however, we cannot always hope that the best classifier always satisfies the fairness constraints. 
Correspondingly, we aim to find a model satisfying the following definition.
\begin{definition}[Solution Characterization] \label{defn:solution}
Given classification loss $L_C(w)$ and adversary loss $L_F(u,w)$, for $\varepsilon, \delta > 0$, a model with parameters $(\h_w, \h_u)$ is an $\varepsilon,\delta$-solution if
\[L_C(\h_w) \leq \min_{w} L_C(w) + \delta,\]
\[L_F(\h_u, \h_w) \geq \max_{u,w} L_F(u,w) - \varepsilon.\]
\end{definition} 
\noindent
Consider the example of logistic regression for classifier and adversary, i.e., $f = \sigma(w^\top x)$ and 
\[L_C = \text{log-loss}(\sigma(w^\top x), y), L_F =  - \text{log-loss}(\sigma(u^\top f(x)), z)\]
%\[L_F =  - \text{log-loss}(\sigma(u^\top f(x)), z)\]
where $\hat{y}$ is the predictions from the classifier. 
The classifier here tries to correctly predict the class label, while the adversary tries to deduce the sensitive attribute information from the classifier output.
This model is similar to the one considered in \cite{zhang2018mitigating} for the Adult dataset.

For this example, the functions $L_C$ and $L_F$ do not have a unique optimizer unless the feature matrix is full-rank. 
Since in general we can expect that the number of samples is greater than the dimension of vectors $x$, there will be multiple optimizers for the above loss function and correspondingly there will be multiple $\varepsilon,\delta$-solutions. 
The same argument holds for any model which uses thresholding or sigmoid-like functions.

\subsection{Model used for Classifier and Fairness Adversary - Statistical Parity}
In this section, we define the model we use for classification and adversary, with the goal of ensuring statistical parity.
We also provide the model and results when the fairness metric is false discovery rate in a later section.

\subsubsection{Classifier}
The model considered is the regularized logistic regression model. For a given weight vector $w \in \R^{n+1}$, 
$f = \sigma(w^\top \hat{x}),$
where $\sigma(\cdot)$ is the sigmoid functions. 

Similar to the structure of Generative Adversarial Networks, we wither add some noise to input of the classifier so as to make it partially randomized or add an additional vector with value 1 to act as the bias. 
Correspondingly, $\hat{x} = [x \text{ } \eta]$, where $\eta$ is either uniformly chosen from $[0,1]$ or is fixed to be 1.

\begin{remark}[Reason for adding noise or 1s vector]
If the sensitive attribute and the class label are highly correlated, the only way for the classifier to satisfy statistical parity is to output a random class label for each data-point. 
In this scenario, the algorithm can make $w_{n+1}$ much larger than other weights ensuring that the output classifier is random. 
We investigate this empirically in Section~\ref{sec:noise}.
As expected, as the dataset becomes more adversarial, the weight given to the final element increases.
\end{remark}

\subsubsection{Classification Loss}
The corresponding classification loss function is  $L_C(w) = \textrm{log-loss}_S(f) + \frac{1}{2} \norm{w}^2.$
The loss function is the standard one for regularized logistic regression.

\subsubsection{Fairness Adversary}
Since we want to ensure statistical parity, we look at how well we can predict the sensitive variable $z$ using the classifier output. Correspondingly, the fairness adversary will be a classifier $g$, where for a particular $d > 0$ and $u \in \R^{d+1}$,
$g = \sigma(u^\top\hat{f}(x)).$
$\hat{f} \in \R^{d+1}$ is the polynomial expansion of $w^\top x$ of degree $d$, i.e., $\hat{f}(x) = [1 \: w^\top x \: (w^\top x)^2 \: \cdots \: (w^\top x)^d]$.
In the rest of the document, we will use $d = 2$ unless explicitly specified.

\subsubsection{Fairness Adversary Loss}
The fairness adversary loss function is
\begin{align*}
L_F(u,w) =  - \textrm{log-loss}_S(g) - \frac{\mu}{2} \left( \sum_{i \in [N] \mid z_i = 0} \frac{w^\top x_i}{\mathbb{P}[G_0]} - \sum_{i \in [N] \mid z_i = 1} \frac{w^\top x_i}{\mathbb{P}[G_1]} \right)^2.
\end{align*}
where $\mathbb{P}[G_j]$ is the probability that the sensitive attribute is $j$ in the training set.
The first part of the loss function corresponds to learning the correlation between the sensitive attribute and the class label.
The second part of the loss function is a regularizer to check whether fairness is satisfied. 
Maximizing the adversary loss would ensure that $f(x_i)$ can be predicted using $z_i$ and that the statistical rate is low.

The reason for choosing such a regularizer is that intuitively, statistical parity will be ensured if $w^\top x_i$ is equally distributed across all groups in the dataset. 
However, since both the sensitive attribute values may not be present equally in the training set, we divide the elements by $\mathbb{P}[G_i]$.

\begin{remark}[Choice of adversary]
The adversary chosen here is different than the one suggested in \cite{zhang2018mitigating}. 
There they design the adversary to predict the sensitive attribute from the classifier output.
However, this will not ensure fairness when the sensitive attribute and class label are highly correlated, or when the dataset has very few elements for a particular sensitive attribute value. 
The model we suggest tries to learn the correlation between sensitive attribute and class label and uses the fairness metric as a regularizer term.
\end{remark}

\section{Algorithms} \label{sec:algorithms}
To solve the min-max problem, we can use an alternating gradient ascent-descent algorithm, which simultaneously aims to minimize $L_C$ and maximize $L_F$. We list below the algorithms we use for our experiments and analysis.

\subsection{Gradient Descent/Ascent with Normal Update} \label{alg:grad_normal_update}
Using a normal gradient descent/ascent algorithm would imply moving in the direction $- \nabla_w L_C$ to minimize $L_C$ and $\nabla_w L_F$ to maximize $L_F$. Combining the two directions with a controlling parameter $\alpha > 0$, we get the algorithm with the following updates at each step.
\[u_{t+1} = u_t + \eta_1 \nabla_{u} L_F, \]
\[w_{t+1} = w_t - \eta_2(\nabla_w L_C - \alpha \cdot \nabla_w L_F),\]
%\,\]
for some $\eta_1, \eta_2, \alpha > 0$.

\begin{figure*}
\centering
  \includegraphics[width=\linewidth]{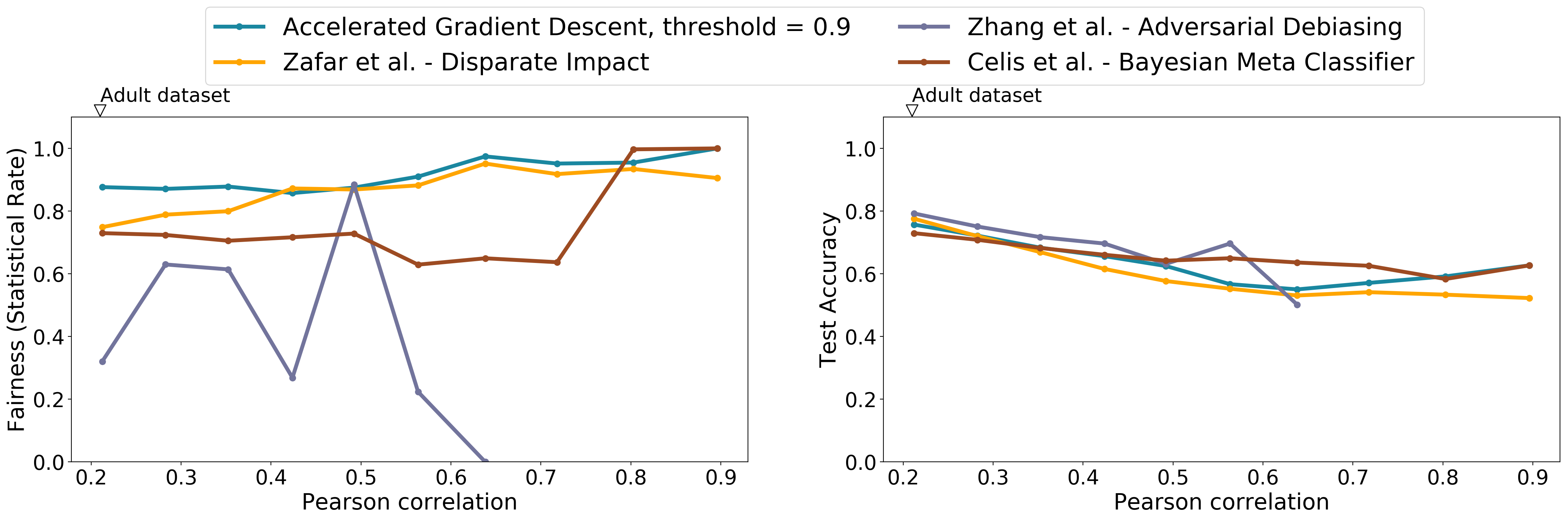}
\caption{Comparison of test accuracy and statistical rate obtained using our algorithms vs related work for varying correlation between the class label and the sensitive attribute.
Note that the left-most point with correlation = 0.21 corresponds to the Adult dataset.
For \cite{zhang2018mitigating}, the datasets with last three correlation coefficients returned all 0s as the predictions and so we do not plot it.}
\label{fig:figure12}
\end{figure*}

\subsection{Algorithm 1 - Normal GD with Modified Update} \label{alg:grad}
In certain cases where the gradient of the fairness loss and the gradient of the classification loss are highly correlated, Algorithm~\ref{alg:grad_normal_update} will not be able to ensure both fairness and accuracy. 
We provide examples and theoretical analysis of such cases in Section~\ref{sec:analysis_without_proj}. 

However, with a simple modification we can ensure that even if gradient of the fairness loss and the gradient of the classification loss are highly correlated, the output classifier is fair.
To that end, we consider the modified update step, where we remove the projection of $\nabla_w L_F$ from $\nabla_w L_C$ from the update.
With an appropriate starting point, at each iteration we use the following update steps.
$$u_{t+1} = u_t + \eta_1 \nabla_{u} L_F,$$
\[w_{t+1} = w_t - \eta_2(\nabla_w L_C - \alpha \cdot \nabla_w L_F - \Pi_{\nabla_w L_F} \nabla_w L_C),\]
for some $\eta_1, \eta_2> 0$ and $\alpha = 1 / \sqrt{t}$.
Though the modified update step we use is inspired by the work of \cite{zhang2018mitigating}, the models and analysis we use are quite different from their work.
\noindent
\textbf{Thresholding: }
Since the problem is a multi-objective optimization problem, we cannot expect to converge to the optimal point (or close) after $T$ iterations.
Therefore, instead of running the algorithm for a large number of iterations, we use a thresholding mechanism.
Given a threshold $\tau > 0$, during training we will record the parameters $w_t,u_t$ for which the training statistical rate $\geq \tau$ and the training accuracy is maximum.
The threshold $\tau$ can be taken as an input and can be considered a way for the user to control the fairness of the output classifier.

\subsection{Algorithm 2 - AGD  with Modified Update} \label{alg:acc_grad}
We modify the earlier alternating gradient ascent-descent algorithm to get an accelerated algorithm. 
The accelerated algorithm we use is inspired by the work of \cite{cohen2018acceleration}, where they provide an improved Accelerated Gradient Descent method that works for noisy gradient oracles. 
Similar accelerated methods of smooth optimization for different kinds of inexact oracles were also given by \cite{d2008smooth, devolder2014first}.

Assume that  $L_C$ is $L_1$-smooth and $L_F$ is $L_2$-smooth. 
With an appropriate starting point, at each iteration we use the following update steps.
\[u_{t} = u_{t-1} + \eta \nabla_{u} L_F,\]
\[g(w) = \nabla_w L_C - \alpha \cdot \nabla_w L_F(w) - \Pi_{\nabla_w L_F(w)} \nabla_w L_C(w)  ,\]
\[p_{t} = \frac{A_{t-1}}{A_t} q_{t-1} + \frac{a_t}{A_t}\nabla \psi(w_{t-1}), \]
\[w_t = w_{t-1} - a_t g(w_{t-1}),\]
\[q_{t} = \frac{A_{t-1}}{A_t} q_{t-1} + \frac{a_t}{A_t}\nabla \psi(w_t), \]
for some $\eta, \alpha > 0$.  
%For this case, we choose $\alpha$ to be a constant. 
The regularizer function $\psi(\cdot)$ is a strongly convex function. We choose $$\psi(w) = \frac{1}{2} \norm{w}^2$$ for our analysis and experiments. $A_t = \sum_{i=1}^t a_i$ and the numbers $\br{a_i}_i$ are chosen as 
$$a_t = \frac{1}{\alpha L_1L_2\sqrt{t}}.$$

\noindent
\textbf{Thresholding: }
Similar to Algorithm~\ref{alg:grad}, we can use the thresholding mechanism here as well, given input $\tau >0$.

\section{Empirical Results} \label{sec:experiments}
We evaluate the performance of this method empirically, and report the classification accuracy and fairness on both a real-world dataset, and on adversarially constructed synthetic datasets. 
We compare different update methods, and also contrast against state-of-the-art algorithms. 

\subsection{Datasets} \label{sec:dataset}
We conduct our experiments on the Adult income dataset \cite{Dua:2017}. 
%\textcolor{red}{cite}
%
This dataset has the demographic information of approximately 45,000 individuals, and class labels that determine whether their income is greater than \$50,000 USD or not.
For the purposes of our simulations, we consider the sensitive attribute of gender, which is coded as binary in the dataset. 

Additionally, we construct adversarial synthetic datasets from the Adult dataset in order to show the wide applicability of the model.
The method of constructing the dataset is related to the choice of adversary, since we want to show that our algorithm performs well even for adversarial datasets.
As we noted before, the algorithm with normal gradient updates (Section~\ref{alg:grad_normal_update}) can perform poorly when the sensitive attribute and the class label are highly correlated (Section~\ref{sec:analysis_without_proj}).
Hence, for a given value of the Pearson correlation coefficient, we generate a synthetic dataset, with feature vectors of the Adult dataset, where the class labels are modified to ensure that correlation coefficient between the class label and the sensitive attribute is the given value.
We generate multiple such datasets for varying correlation coefficients and test our algorithm and other state-of-the-art algorithms on these datasets.

\subsection{Performance as a Function of Time}
We first look at the training performance of Algorithm~\ref{alg:acc_grad} and Algorithm~\ref{alg:grad} on the Adult and synthetic datasets.
 
\subsubsection{Adult Dataset: }
For these plots, the parameters chosen for the algorithm are: learning rate, $\eta_1 = \eta_2 = 0.1$, $\alpha = 0.1 / \sqrt{t}$ and number of iterations $= 100$. 
The training accuracy and training statistical rate as presented in the first plot of Figure~\ref{fig:compare_agd_ngd_adult}. 
As can be seen from the figure, for the Adult dataset the algorithm eventually converges to a point of high accuracy and high statistical rate using Algo~\ref{alg:grad}. When using Algo~\ref{alg:acc_grad}, we get similar high accuracy, but Algo~\ref{alg:acc_grad} converges to the point of high accuracy and high fairness faster. 
However, the final statistical rate obtained using Algorithm~\ref{alg:acc_grad} is better.

\subsubsection{Synthetic Dataset: }
The Pearson correlation between the class label and the sensitive attribute for the synthetic dataset used here is 0.5.
Once again, the parameters chosen for the algorithm are: learning rate, $\eta_1 = \eta_2 = 0.1$, $\alpha = 0.1 / \sqrt{t}$ and number of iterations $= 100$. 
The training accuracy and training statistical rate as presented in Figure~\ref{fig:compare_agd_ngd_syn}. 
As can be seen from the figure, for the synthetic dataset both algorithms eventually converges to a point of high accuracy. 
Once again, Algo~\ref{alg:acc_grad} converges faster than Algo~\ref{alg:grad}.

\subsection{Modified vs. Normal Update Steps} \label{sec:compare_normal_mod_update_expt}
We compare the modified and normal gradient update steps. 
Figure~\ref{fig:figure9} shows how the accuracy and statistical rate varies for both the algorithms as the correlation between the class label and sensitive attribute changes.
We run the AGD Algorithm~\ref{alg:acc_grad} in both settings, with and without threshold the threshold parameter $\tau = 0.8$. 
It is clear from the figures that the modified gradient update algorithms achieve much better fairness than the normal gradient update algorithms.

To further show the importance of removing the projection, we also show plots of the statistical rate across the training iterations. 
Figure~\ref{fig:figure1} shows the plots for that setting. 
From Figure~\ref{fig:compare_agd_ngd_adult} and Figure~\ref{fig:compare_agd_ngd_adult}, it is clear that for both Adult and synthetic datasets, the algorithm with modified update performs much better than the algorithm with normal update.
When using AGD, the statistical rate is never high for normal gradient update algorithm.
While for Normal Gradient Descent, the statistical rate is high for normal gradient update algorithm during the initial iterations, it drops very quickly and does not allow the accuracy to be high along with high statistical rate.
Finally, Figure~\ref{fig:compare_corr_train} shows that for all datasets with varying correlation coefficients, the fairness at timestep 50 of training is never high when using the gradient descent algorithm with normal updates.

\begin{figure}[t]
\centering
  \includegraphics[width=0.5\linewidth]{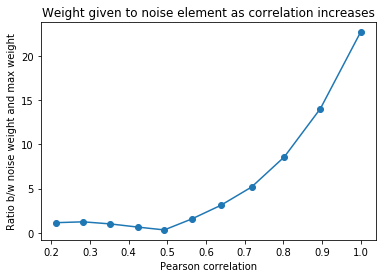}
\caption{Relative Weight given to noise element as correlation increases.}
\label{fig:figure7}
\end{figure}

\subsection{Comparison Against the State of the Art}
We now compare the performance of our proposed methods against the state-of-the-art fair classification techniques. We vary the correlation in the synthetic dataset and report the test accuracy, test statistical rate. 
The synthetic datasets are constructed as discussed in Section~\ref{sec:dataset}.  The threshold $\tau$ is set to 0.9.
We compare our algorithm with the fair classifiers of \cite{zafar2017fairness} \footnote{\url{github.com/mbilalzafar/fair-classification}}, \cite{zhang2018mitigating}  \footnote{\label{note1} \url{github.com/IBM/AIF360}} \footnote{Note that NGD with modified update can be considered the \cite{zhang2018mitigating} implementation for our model and we already showed that AGD with modified update converges faster than it.} and \cite{celis2018classification} \textsuperscript{\ref{note1}}, and present the results in Figure~\ref{fig:figure12}.

We observe that using Accelerated Gradient Descent with threshold 0.9, we get higher or comparable fairness (statistical rate) for all datasets, albeit with a small loss to accuracy.
In particular, when the statistical rate obtained by our algorithm is smaller (for example, when correlation $\sim 0.65$), the accuracy of the classifier is  higher.
Furthermore, by increasing the threshold parameter, our algorithm can always be forced to achieve higher fairness; indeed, from Figure~\ref{fig:figure1}, we know that our algorithm can achieve perfect statistical {parity} during the training process for all synthetic datasets.

We also construct an adversarial model to ensure high false discovery rate of the output classifier.
The model and the empirical comparison with other algorithms are presented in Section~\ref{sec:fdr}.
The empirical results show that our model achieves higher false discovery rate than other algorithms, while the accuracy is comparable for most datasets and slightly lower in other datasets.

\begin{figure*}[t]
\centering
  \includegraphics[width=\linewidth]{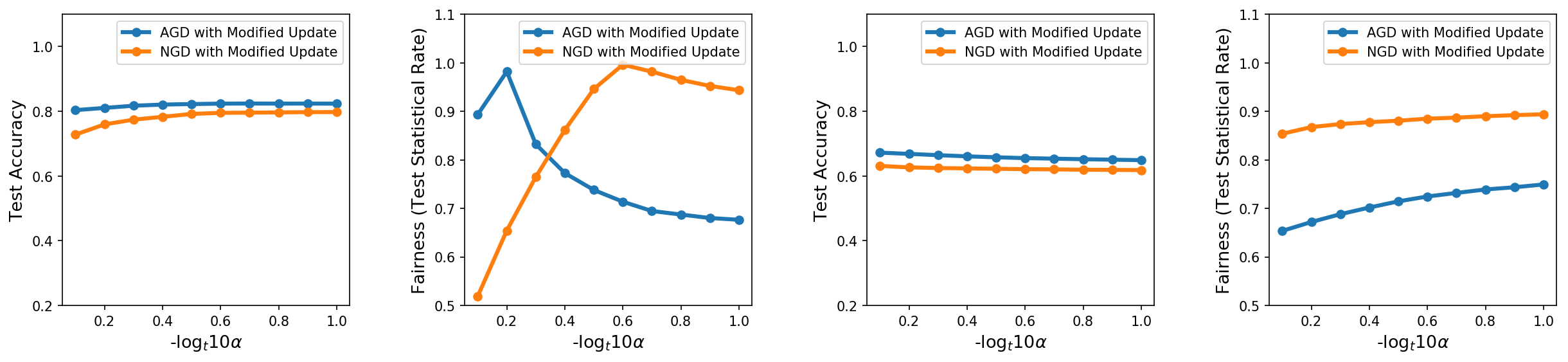}
  \subfloat[Adult Dataset. Correlation = 0.21.]{\hspace{.5\linewidth}}
\subfloat[Synthetic Dataset. Correlation = 0.80.]{\hspace{.5\linewidth}}
\caption{ The two plots for each dataset represent the Test accuracy and Statistical rate using the Gradient Descent for different decreasing functions of $\alpha$. AGD is Accelerated Gradient Descent and NGD is Normal Gradient Descent.}
\label{fig:figure14}
\end{figure*}

\subsection{Other Experiments}

\subsubsection{Importance of noise} \label{sec:noise}
As mentioned earlier, we add a noise or 1s element to the feature vector of each datapoint. As correlation between the class label and the sensitive attribute increases, the only way to ensure high statistical parity is to make the classifier either random or output all 1s. Correspondingly, we hope that the adversary feedback pushes the classifier to make the weight given to the noise element larger.

In this section, we measure this observation. The threshold $\tau$ is set to be 0.8. We measure the ration between the weight given to the noise element and the maximum weight given to any other element in the feature vector, i.e.,
\[\frac{w_\eta}{\max(w_{-1})},\]
where $w_\eta$ is the weight of the noise element and $w_{-1}$ is the weight vector without the noise element. The plot of the ratio against the the correlation in the synthetic dataset is presented in Figure~\ref{fig:figure7}. As we can see from the figure, the weight of the noise element increases as the correlation increases, which is the expected behaviour.

\subsubsection{Changing parameter $\alpha$} \label{sec:alpha_expt}
Recall that in the modified gradient update step, we use
\[g(w) = \nabla_w L_C - \alpha \cdot \nabla_w L_F - \Pi_{\nabla_w L_F} \nabla_w L_C.\]
In this section, we look at the affect of change of $\alpha$ on the accuracy and statistical parity for both Adult dataset and Synthetic dataset.
 Ideally $\alpha$ is chosen as a decaying function of iteration $t$. 
We set it as 
\[\alpha = O\left( \frac{1}{t^p}\right) , \text{ where } p \in \br{0.1, 0.2, \dots, 1.0}.\]
The results are presented in Figure~\ref{fig:figure14}.
As we can see, the test statistical rate and accuracy do not really change with $\alpha$ for both algorithms.

\section{Theoretical Results} \label{sec:theory}
\subsection{Update Step Without the Projection Term} \label{sec:analysis_without_proj}
We look at the gradient descent algorithm with normal gradient updates (Algorithm~\ref{alg:grad_normal_update}) and provide counter-examples for the case when it may not be able to ensure fairness.
Algorithm~\ref{alg:grad_normal_update} uses the following updates in each iteration (assuming $\eta_1 = \eta_2$)  
\[u_{t+1} = u_t + \eta \nabla_{u} L_F, \]
\[w_{t+1} = w_t - \eta(\nabla_w L_C - \alpha \cdot \nabla_w L_F).\]
Assuming that the adversarial loss function is concave,
\begin{align*}
L_F(w_{t+1}, u_{t+1}) - L_F(w_t, u_t) &\leq \inp{\nabla_u L_F(u_t), u_{t+1} - u_{t}} +  \inp{\nabla_w L_F(w_t), w_{t+1} - w_{t}} \\
%&= \eta \norm{\nabla_u L_F(u_t)}^2\\ &- \eta \inp{\nabla_w L_F(w_t), \nabla_w L_C(w_t) - \alpha \cdot \nabla_w L_F(w_t)} \\
&\leq \eta \norm{\nabla_u L_F(u_t)}^2 + \alpha\eta\norm{\nabla_w L_F(w_t)}^2 - \eta \inp{\nabla_w L_F(w_t), \nabla_w L_C(w_t)}
\end{align*}
For $\alpha < 1$, if $\inp{\nabla_w L_F, \nabla_w L_C} > \norm{\nabla_u L_F}^2 + \alpha\norm{\nabla_w L_F}^2$ then the gradient update step leads away from the optimal point. 
A simple example of when this can happen is the following: suppose that $L_C(w)$ is a simple logistic regression log-loss function and $L_F(u,w)$ is a regularizer function controlling statistical parity, i.e., $f = \sigma(w^\top x) $ and
\[ L_F(w) = \left( \sum_{(x_i, z_i) \mid z_i = 0} \frac{\sigma(w^\top x_i)}{\mathbb{P}[G_0]} - \sum_{(x_i, z_i) \mid z_i = 1} \frac{\sigma(w^\top x_i)}{\mathbb{P}[G_1]} \right)^2.\]
The adversary here is not a classification/learning problem and is instead a simple regularizer function. 
Suppose the dataset given is skewed in such a way that most of the datapoints with sensitive attribute $z = 0$ have class label $y=0$ as well. 
In this case to achieve high accuracy, statistical parity has to be low. 
Therefore, $\inp{\nabla_w L_F, \nabla_w L_C}$ will be large and $\norm{\nabla_u L_F}$ will be zero.
In such a case, $L_F$ will reduce for all iterations and fairness will never be ensured. 

However, this will not happen if we remove the projection of $\nabla_w L_F$ from $\nabla_w L_C$, since then $$\inp{\nabla_w L_F, \nabla_w L_C - \alpha \cdot \nabla_w L_F - \Pi_{\nabla_w L_F} \nabla_w L_C} = -\alpha \norm{\nabla_w L_F}^2.$$
\begin{remark}
Note that the above scenario can happen for any kind of classification model where the direction of the gradient loss function and the direction of the gradient of regularizer are highly correlated for the given dataset.
\end{remark}

\subsection{The Modified Update Step} \label{sec:mod_theory}
We analyze the modified gradient update step (with the projection term) in the context of multi-objective optimization.

\subsubsection{Analysis of Algorithm~\ref{alg:grad}}
We first look at the normal gradient descent/ascent method using the modified update step. 
The theorems in this section quantify the number of iterations required to ensure fairness and the classification loss achieved after those many iterations.
We will assume that the gradients satisfy the following assumption on their norm-values. 
\begin{definition}[Bounded Gradient] \label{defn:bg} There exists $G>0$ such that for all $w$,
$\frac{1}{G} \norm{\nabla_w L_F(w)} \leq \norm{\nabla_w L_C(w)} \leq G \norm{\nabla_w L_F(w)}.$
\end{definition}

\begin{theorem}[Ensuring Fairness] \label{thm:fairness_bound}
Let $L_F$ be an $L$-smooth function and assume that the $L_F$ and $L_C$ satisfy Assumption~\ref{defn:bg}. Let $\alpha = t^{-c}$, for $c < 0.5$, in Algorithm~\ref{alg:grad}. Let $$D := \max_{u,w} \sqrt{\norm{u-u^*}^2 + \norm{w-w^*}^2}$$ and $\varepsilon > 0$. If the algorithm uses the learning rates 
$$\eta_1 = \frac{1}{2L} \textrm{ and }\eta_2 = \frac{\alpha}{2L(1+G^2)},$$
then after $T$ iterations, where
\[T = O\left( \left( \frac{LD^2(1+G^2)(1-2c)}{\varepsilon} \right)^{\frac{1}{1-2c}} \right),\]
we will have $ L_F(w^*, u^*) - L_F(w_t, u_t) \leq \varepsilon$.
\end{theorem}
\noindent
The above theorem gives us a convergence bound for $L_F$. 
\begin{proof}
Since $-L_F$ is convex and $L$-smooth, at the $(t+1)$-th timestep,
\begin{align*}
- &L_F(w_{t+1}, u_{t+1}) + L_F(w_t, u_t) \\&\leq - \inp{\nabla_u L_F, u_{t+1} - u_{t}} -  \inp{\nabla_w L_F, w_{t+1} - w_{t}} + L \cdot (\norm{u_{t+1} - u_t}^2 + \norm{w_{t+1} - w_t}^2)\\
&= - \eta_1 \norm{\nabla_u L_F}^2 + \eta_1^2 L \norm{\nabla_u L_F}^2 + \eta_2 \inp{\nabla_w L_F, \nabla_w L_C - \alpha \cdot \nabla_w L_F - \Pi_{\nabla_w L_F} \nabla_w L_C}\\
& + \eta_2^2 L \norm{\nabla_w L_C - \alpha \cdot \nabla_w L_F - \Pi_{\nabla_w L_F} \nabla_w L_C}^2\\
& = - \eta_1 \norm{\nabla_u L_F}^2 + \eta_1^2 L \norm{\nabla_u L_F}^2 - \alpha \eta_2 \norm{\nabla_u L_F}^2 + \alpha^2 \eta_2^2  L \norm{\nabla_w L_F}^2 + \eta_2^2 L \norm{\nabla_w L_C}^2\\
& \leq - \eta_1 \norm{\nabla_u L_F}^2 + \eta_1^2 L \norm{\nabla_u L_F}^2 - \alpha \eta_2 \norm{\nabla_u L_F}^2 + \eta_2^2  L (1 + G^2) \norm{\nabla_w L_F}^2.
\end{align*}
Using 
\[\eta_1 = \frac{1}{2L} \textrm{ and }\eta_2 = \frac{\alpha}{2L(1+G^2)}\]
we get,
\begin{align*}
L_F(w_t, u_t) - L_F(w_{t+1}, u_{t+1}) &\leq -\frac{1}{4L} \norm{\nabla_u L_F}^2 - \frac{\alpha^2}{4L(1+G^2)} \norm{\nabla_w L_F}^2 \\
& \leq - \frac{\alpha^2}{4L(1+G^2)}  ( \norm{\nabla_u L_F}^2 + \norm{\nabla_w L_F}^2) \\
&= - \frac{\alpha^2}{4L(1+G^2)}  \norm{\nabla L_F}^2.
\end{align*}
Let $R_F(t) := L_F(w^*, u^*) - L_F(w_t, u_t)$. Then using the above inequality,
\[R_F(t) - R_F(t+1) \geq \frac{\alpha^2}{4L(1+G^2)}  \norm{\nabla L_F}^2.\]
Also, by the concavity of $L_F$, we get,
\begin{align*}
R_F(t) &= L_F(w^*, u^*) - L_F(w_t, u_t)\\
& \leq \inp{\nabla_w L_F(w_t, u_t), w^* - w_t} \\ & + \inp{\nabla_u L_F(w_t, u_t), u^* - u_t}\\
& \leq \norm{\nabla L_F} \cdot D.
\end{align*}
Therefore,
\[ \norm{\nabla L_F} \geq \frac{R_F(t)}{D},\]
and
\[R_F(t+1) - R_F(t) \leq -\frac{\alpha^2}{4LD^2(1+G^2)}  R_F(t)^2.\]
We can analyze the time-continuous version of the above equation. It leads to following differential equation.
\[\dv{R_F}{t} = -\frac{\alpha^2}{4LD^2(1+G^2)}  R_F^2.\]
We use $\alpha = t^{-c}$, where $c < 0.5$ and let $\beta = 4LD^2(1+G^2)$. Also $R_F(0) = LD^2$ Then we get,
\begin{align*}
&\frac{\textrm{d}R_F}{R_F^2} = - \frac{\textrm{d}t}{\beta t^{2c}}\\
\implies &\frac{1}{LD^2} - \frac{1}{R_F(T)}  = - \frac{1}{(1-2c)\beta} T^{1-2c}.
\end{align*}
Since we want $R_F(T)$ to be small, we look at the iterations $T$ required to reduce it to $\varepsilon$.
\begin{align*}
& \frac{1}{\varepsilon} - \frac{1}{LD^2}  = \frac{1}{4LD^2(1+G^2)(1-2c)} T^{1-2c}\\
\implies& T = O\left( \left( \frac{LD^2(1+G^2)(1-2c)}{ \varepsilon} \right)^{\frac{1}{1-2c}} \right). 
\end{align*}

\end{proof}

\noindent
Using the above number of iterations, we can quantify how far $L_C(w^T)$ is from $L_C(w^*)$. 
\begin{theorem}[Price of Fairness] \label{thm:price_of_fairness}
Let $L_C$ also be an $L$-smooth and $G_C$-gradient bounded function and assume that the $L_F$ and $L_C$ satisfy Assumption~\ref{defn:bg}. Let $\alpha = t^{-c}$, for $c < 0.5$, in Algorithm~\ref{alg:grad}. Let $$D := \max_{u,w} \sqrt{\norm{u-u^*}^2 + \norm{w-w^*}^2}$$ and $\varepsilon > 0$. The algorithm uses the learning rates 
$$\eta_1 = \frac{1}{2L} \textrm{ and }\eta_2 = \frac{\alpha}{2L(1+G^2)}.$$
Let the number of iterations $T$ be the same as obtained in Thm~\ref{thm:fairness_bound}.
%\[\text{\footnotesize $T = O\left( \left( \frac{LD^2(1+G^2)(1-2c)}{ \varepsilon} \right)^{\frac{1}{1-2c}} \right)$}.\]
Suppose that the normal gradient descent method, using gradient updates $\nabla L_C(w)$, is $\delta$-close to the minimizer of $L_C$ after $T$ iterations, where
$\delta \leq 2D \cdot \sqrt{G_CG}.$
Then after $T$ iterations, we will have
$L_C(w_t) - L_C(w^*) \geq \max \br{ \delta,  \delta'}  ,$ 
%\begin{align*}
\[\frac{1}{\delta'} = \frac{1}{LD^2} + \frac{1}{8(1 - c)^2 \varepsilon} O\left( \left( \frac{LD^2(1+G^2)(1-2c)}{ \varepsilon} \right)^{\frac{c}{1-2c}} \right)\]
%\end{align*}
\end{theorem}
The above theorem gives us an estimation of the \textit{price of fairness} incurred by our algorithm, i.e., to achieve perfect fairness through the above model, the least amount of classification loss we have to sacrifice, compared to the minimum classification loss.
\begin{proof}
We obtain the price of fairness bound by analyzing the convergence of $L_C$.
\begin{align*}
&L_C(w_{t+1}) - L_C(w_t)\\ &\leq \inp{\nabla_w L_C, w_{t+1} - w_{t}} + \norm{w_{t+1} - w_t}^2\\
&= - \eta_2 \inp{\nabla_w L_C, \nabla_w L_C - \alpha \cdot \nabla_w L_F - \Pi_{\nabla_w L_F} \nabla_w L_C} + \eta_2^2 L \norm{\nabla_w L_C - \alpha \cdot \nabla_w L_F - \Pi_{\nabla_w L_F} \nabla_w L_C}^2\\
& = -\eta_2 \norm{\nabla_w L_C}^2 + \eta_2^2L\norm{\nabla_w L_C}^2 + \eta_2^2 \alpha^2 L\norm{\nabla_w L_F}^2 + \eta_2 \inp{\nabla_w L_C, \alpha \cdot \nabla_w L_F + \Pi_{\nabla_w L_F} \nabla_w L_C}\\
&\leq -\eta_2 \norm{\nabla_w L_C}^2 + \eta_2^2L(1+G^2)\norm{\nabla_w L_C}^2 + \eta_2 \inp{\nabla_w L_C, \alpha \cdot \nabla_w L_F + \Pi_{\nabla_w L_F} \nabla_w L_C}\\
&\leq -\frac{\alpha}{4L(1+G^2)}  \norm{\nabla_w L_C}^2 + \frac{2\alpha}{4L(1+G^2)} \inp{\nabla_w L_C, \nabla_w L_F}.
\end{align*}
Let $R_C(t) := L_C(w_t) - L_C(w^*)$. Then using the above inequality,
\begin{align*}
R_C(t) - R_C(t+1) \geq \frac{\alpha}{4L(1+G^2)}  \norm{\nabla_w L_C}^2 - \frac{2\alpha}{4L(1+G^2)} \inp{\nabla_w L_C, \nabla_w L_F}.
\end{align*}
Also, by the concavity of $L_C$, we get,
\begin{align*}
R_C(t) &= L_C(w^*) - L_C(w_t)\\
& \leq \inp{\nabla_w L_C(w_t, u_t), t_t - w^*}\\
&  \leq \norm{\nabla_w L_C} \cdot D.
\end{align*}
Therefore,
\[ \norm{\nabla_w L_C} \geq \frac{R_C(t)}{D},\]
and
\begin{align*}
R_C(t+1) - R_C(t) \leq  - \frac{\alpha}{4LD^2(1+G^2)}  R_C(t)^2 + \frac{2\alpha}{4L(1+G^2)} \inp{\nabla_w L_C, \nabla_w L_F}.
\end{align*}

Note that we want $R_C(t+1)$ to be smaller than $R_C(t)$ for all $t$. Infact we want to find the number of iterations in which it reduces to small $\delta > 0$. Correspondingly, we can assume that
\[R_C(t) \geq \delta \implies \frac{R_C(t)}{\delta} \geq 1.\] 
Therefore,
\begin{align*}
R_C(t+1) - R_C(t) &\leq - \frac{\alpha}{4LD^2(1+G^2)}  R_C(t)^2  + \frac{2\alpha}{4L(1+G^2)\delta^2} \inp{\nabla_w L_C, \nabla_w L_F} R_C(t)^2\\
& = - \frac{R_C(t)^2}{4L(1+G^2) t^c} \left( \frac{1}{D^2}  - \frac{2}{\delta^2} \inp{\nabla_w L_C, \nabla_w L_F} \right).
\end{align*}
Since $L_C$ is $G_C$-gradient bounded, we get
\[R_C(t+1) - R_C(t) \leq - \frac{R_C(t)^2}{4L(1+G^2)t^c} \left( \frac{1}{D^2}  - \frac{2G_CG}{\delta^2} \right).\]
Once again, we formulate it as a differential equation and solve it.  Note that $R_C(0) = LD^2.$
\[\dv{R_C}{t} =  - \frac{R_C^2}{t^c} \cdot  \beta' \implies  \frac{1}{R_C(T)} = \frac{1}{LD^2} + \frac{\beta'}{1 - c} T^{1-c}.\]
Substituting the value of $\beta'$, we get
\[\frac{1}{R_C(T)} = \frac{1}{LD^2} + \frac{1}{4L(1+G^2)(1 - c)}  \left( \frac{1}{D^2} - \frac{2G_CG}{\delta^2} \right) T^{1-c}.\]
Since 
\[\delta \leq 2D \cdot \sqrt{G_CG} \implies  \frac{1}{2D^2} \leq \frac{2G_CG}{\delta^2},  \]
we get
\[\frac{1}{R_C(T)} \leq \frac{1}{LD^2} + \frac{1}{8LD^2(1+G^2)(1 - c)} T^{1-c}.\]
Substituting the value of $T$ from the previous theorem, we get
\begin{align*}
\frac{1}{R_C(T)} \leq \frac{1}{LD^2} + &\frac{1}{8LD^2(1+G^2)(1 - c)} O\left( \left( \frac{LD^2(1+G^2)(1-2c)}{\varepsilon} \right)^{\frac{1-c}{1-2c}} \right) \\
&\leq \frac{1}{LD^2} + \frac{1}{8(1 - c)^2 \varepsilon} O\left( \left( \frac{LD^2(1+G^2)(1-2c)}{ \varepsilon} \right)^{\frac{c}{1-2c}} \right)
\end{align*}
\end{proof}

\subsubsection{Analysis of Algorithm~\ref{alg:acc_grad}}
Using the results of \cite{cohen2018acceleration}, we obtain the following theoretical bounds on the classification loss and adversary loss achieved by Algorithm~\ref{alg:acc_grad}.
They prove the following theorem regarding their algorithm.

\begin{theorem}[Convergence of AGD with noisy gradient \cite{cohen2018acceleration}] \label{thm:cohen}
Let $f$ be an $L$-smooth function and let $w_0$ be an arbitrary initial point. Suppose the noisy gradient of $f$ is of the form 
$g(w_t) = \nabla_w f(w_t) + \lambda_t,$
where $\lambda_t$ is the noise at iteration $t$. For $\psi(w) = \frac{1}{2}\norm{w}^2$, if sequences $w_t, p_t, q_t$ evolve according to AGD+ \cite{cohen2018acceleration}, $R_{w^*} = \max_w \norm{w - w^*}$ and $\frac{a_t^2}{A_t} \leq \frac{1}{L}$, then
\[f(q_t) - f(w^*) \leq \frac{\mathcal{D}(w_0, w^*)}{A_t} + R_{w^*} \frac{\sum_{i=1}^t a_i \norm{\lambda_t}}{A_t},\]
where $\mathcal{D}(w_1,w_2) = \frac{1}{2} \norm{w_1}^2 + \frac{1}{2} \norm{w_2}^2 - 2 \inp{w_1, w_2}$.
\end{theorem}
%\vspace{-0.1in}
\noindent
Using Theorem~\ref{thm:cohen}, we derive bounds on how close we are to the optimal point after $T$ iterations. 
%
%\vspace{-0.1in}
\begin{theorem}[Convergence of Algorithm~\ref{alg:acc_grad}] \label{thm:thm_agd}
Suppose $L_C$ is an $L_1$-smooth function, $G_C$ gradient-bounded function and $L_F$ is an $L_2$-smooth, $G_F$-gradient-bounded function. Let $u_0, w_0$ be an arbitrary initial point. If sequences $u_t, w_t, p_t, q_t$ evolve according to Algorithm~\ref{alg:acc_grad} with constant $\alpha = \max\br{\frac{1}{L_1}, \frac{1}{L_2}}$ and $R_{u^*, w^*} = \max_{u,w} \sqrt{\norm{u - u^*}^2 + \norm{w - w^*}^2}$, then after $T$ iterations, %
\begin{align*}
&L_C(q_T) - L_C(w^*) \leq \frac{\mathcal{D}(w_0, w^*)}{A_T} + 2R_{w^*}G_F, \\
&L_F(u^*, w^*) - L_F(u_T, q_T) \leq \frac{\mathcal{D}(w_0, w^*) + D(u_0,u^*) }{A_T} + \frac{2R_{u^*, w^*}G_C}{\alpha},
\end{align*}
where $\mathcal{D}(w_1,w_2) = \frac{1}{2} \norm{w_1}^2 + \frac{1}{2} \norm{w_2}^2 - 2 \inp{w_1, w_2}$.
\end{theorem}
\begin{proof}
The proof requires an application of the convergence bounds of \cite{cohen2018acceleration} for both $L_C$ and $L_F$. While proving the bound on $L_C$ is straight-forward, proving the bound on $L_F$ requires taking into account the $\alpha$-factor in the gradient update and the numbers $\br{a_t}_t$.

We first prove the bound on $L_C(q_T)$. For the updates in Algorithm~\ref{alg:acc_grad},
\[a_t = \frac{1}{\alpha L_1L_2\sqrt{t}} \leq \frac{1}{L_1 \sqrt{t}} \leq \frac{1}{L_1}.\]
Therefore, it satisfies the condition
\[\frac{a_t^2}{A_t} \leq a_t \leq \frac{1}{L_1}.\]
From the context of optimizing $L_C$, in Algorithm~\ref{alg:acc_grad}, we have
\[\lambda_t = -\left(\alpha \cdot \nabla_w L_F(w_t) + \Pi_{\nabla_w L_F(w_t)} \nabla_w L_C(w_t)\right). \]

Therefore, $\norm{\lambda_t} \leq 2 \norm{\nabla_w L_F} \leq 2G_F$.
Plugging in these values in Thm~\ref{thm:cohen}, we get
\[L_C(q_T) - L_C(w^*) \leq \frac{\mathcal{D}(w_0, w^*)}{A_T} + 2R_{w^*}G_F.\]

Next we prove the bound on $L_F(u_T, q_T)$. 
From the context of optimizing $L_F$, in Algorithm~\ref{alg:acc_grad}, we have
\[a_t' = \frac{1}{L_1L_2\sqrt{t}} \leq \frac{1}{L_2 \sqrt{t}} \leq \frac{1}{L_2}.\]
Note that the numbers (in this case $a_t'$) for optimizing $L_F$ are different than $a_t$ because the $\alpha$-factor on $\nabla_w L_F$ in the gradient is removed by the $1/\alpha$-factor in $a_t$.
Also in this case, the function $\psi'(w) = \frac{\alpha}{2} \norm{w}^2$ and hence the Bregman divergence corresponding to this function $\mathcal{D}'(w_1, w_2) = \alpha\mathcal{D}(w_1,w_2)$.
Finally the noise part $\lambda_t'$ here is
\[\lambda_t' = \frac{1}{\alpha} (\nabla_w L_F(w_t) - \Pi_{\nabla_w L_F(w_t)} \nabla_w L_C(w_t)). \]
Therefore,
$\norm{\lambda_t'} \leq \frac{G_C}{\alpha}$,
and the bound on $L_F(u_T, q_T)$ is
\begin{align*}
L_F(u^*, w^*) - L_F(u_T, q_T) &\leq \frac{\alpha (\mathcal{D}(w_0, w^*) + D(u_0,u^*) )}{ \alpha A_T} + \frac{2R_{u^*, w^*}G_C}{\alpha}\\
&= \frac{\mathcal{D}(w_0, w^*) + D(u_0,u^*) }{A_T} + \frac{2R_{u^*, w^*}G_C}{\alpha}.
\end{align*}

This completes the proof of Theorem~\ref{thm:thm_agd}.
\end{proof}
The theorem also tells us that, by our earlier solution characterization in Definition~\ref{defn:solution}, after $T$ iterations the model obtained is a $$\left( \frac{\mathcal{D}(w_0, w^*) + D(u_0,u^*) }{A_T} + 2R_{u^*, w^*}G_C, \frac{\mathcal{D}(w_0, w^*)}{A_T} + 2R_{w^*}G_F \right) \textrm{-solution.} $$

\begin{figure*}[t]
\centering
  \includegraphics[width=\linewidth]{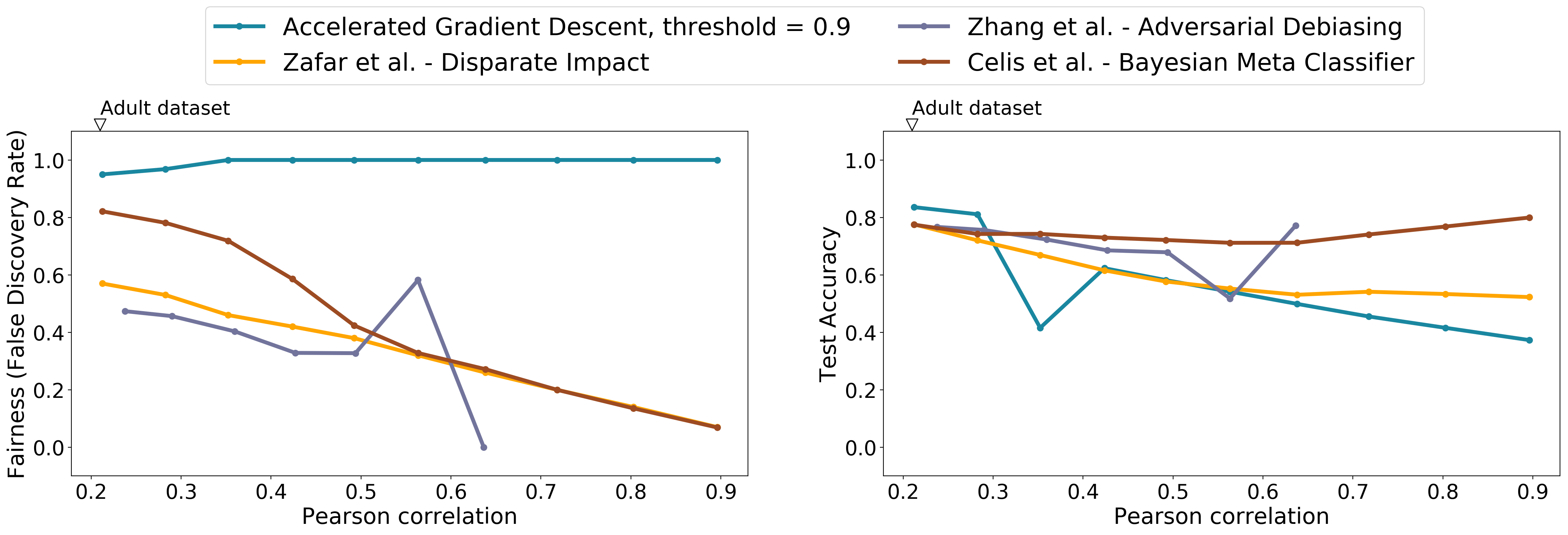}
\caption{Comparison of test accuracy and false discovery rate obtained using different algorithms for varying correlation between the class label and sensitive attribute.
Note that the left-most point with correlation = 0.21 corresponds to the Adult dataset.
For few of the correlation values, the implementations of \cite{zhang2018mitigating} and \cite{zafar2017fairness} gave either errors or all 0s as output. Those datapoints are not included in the figure.} 
\label{fig:figure13}
\end{figure*}

\section{False Discovery Rate} \label{sec:fdr}
In this section, we suggest and analyze a model to find a classifier that ensures false discovery parity.

\subsection{Model}
The classifier and the classification remains the same while ensuring false discovery parity (Definition~\ref{def:fdr}).
The main difference here is the adversary.
Note that in this case, the adversary has access to class label $y$ as well and the adversary model and loss function change accordingly.

\subsubsection{Fairness Adversary}
Since we want to ensure statistical parity, we look at how well we can predict the sensitive variable $z$ using the classifier output and $y$. Correspondingly, the fairness adversary will be a classifier $g$, where for a particular $d > 0$ and $u \in \R^3$,
$g = \sigma(u^\top[1 \; f(x) \; y]).$

%Similarly, the adversary will be a classifier $g$, where for a particular $u_1, u_2 \in \R$,
%\[g = \sigma(u_1 (w^\top x + b) + u_2).\]

\subsubsection{Fairness Adversary Loss}
The fairness adversary loss function is
%\[L_F = - \sum_{i \in [N]} \left( f(x_i) \log g(z_i) + (1 - f(x_i)) \log (1 - g(z_i)) + \left(\frac{\sigma(u^\top\hat{z_i})}{\mathbb{P}[Z = z_i]}\right)^2\right).\]
%
\begin{align*}
L_F(u,w) =  - \textrm{log-loss}_S(g) - \frac{\mu}{2} \left( \sum_{\substack{i \in [N] \\ \mid y_i = 0, \\ z_i = 1}} w^\top x_i \cdot \sum_{\substack{i \in [N]\\ \mid z_i = 0}} w^\top x_i - \sum_{\substack{i \in [N] \\ \mid y_i = 0, \\ z_i = 0}} w^\top x_i \cdot \sum_{\substack{i \in [N]\\ \mid z_i = 1}} w^\top x_i  \right)^2.
\end{align*}
where $\mathbb{P}[G_j]$ is the probability that the sensitive attribute is $j$ in the training set.
The second part of the loss function once again is a regularizer to check whether fairness is satisfied. 

\subsection{Empirical Results}
The accuracy and false discovery rate achieved by our algorithm, along with comparisons with other algorithms, is presented in Figure~\ref{fig:figure13}.
From the figure, it can be seen that our algorithms achieve high false discovery rate in most cases and comparable accuracy with respect to other algorithms.

\section{Limitations and Future Work}
The models and algorithms presented in this work aim to exploit the paradigm of adversarial learning to ensure fairness but there is scope for further research.
The algorithm uses modified gradient updates to find the optimal fair and accurate model. 
Similar to the work of \cite{agarwal2018reductions}, other saddle point techniques can also be compared to our approach.
From a game-theoretic perspective, the problem is equivalent to finding an approximate Nash equilibrium \cite{freund1996game}, and existing work  can be explored to obtain better algorithms.
%

%In this paper, as well other adversarial fairness papers, the notion of fairness used is always \textit{group-fairness}, where the classifier should be fair with respect to every group.
%
The model of adversarial learning should also be explored in the context of ensuring \textit{individual fairness}, as considered in \cite{dwork2012fairness, zemel2013learning}, instead of just \textit{group-fairness} metrics considered here.
The theorems for both the algorithms can be improved, including removing the assumptions on $\alpha$. 
As shown by experiments in Section~\ref{sec:alpha_expt}, the performance of the algorithm is not significantly affected by changing $\alpha$.
%, but for our theoretical analysis, the assumptions are required to get a reasonable convergence bound.
%
Finally, 
%as discussed earlier, the model can be extended to other types of fairness metrics. 
%
as part of future work, similar adversarial models should be constructed for other metrics, continuous fairness features and multi-valued sensitive attributes and labels.

\bibliographystyle{plain}
\bibliography{references}

\begin{thebibliography}{10}

\bibitem{agarwal2018reductions}
Alekh Agarwal, Alina Beygelzimer, Miroslav Dud{\'\i}k, John Langford, and Hanna
  Wallach.
\newblock A reductions approach to fair classification.
\newblock {\em arXiv preprint arXiv:1803.02453}, 2018.

\bibitem{compas}
Julia Angwin, Jeff Larson, Surya Mattu, and Lauren Kirchner.
\newblock \url{https://github.com/propublica/compas-analysis}, 2016.

\bibitem{berk2009role}
Richard Berk.
\newblock The role of race in forecasts of violent crime.
\newblock {\em Race and social problems}, 1(4):231, 2009.

\bibitem{calmon2017optimized}
Flavio Calmon, Dennis Wei, Bhanukiran Vinzamuri, Karthikeyan~Natesan
  Ramamurthy, and Kush~R Varshney.
\newblock Optimized pre-processing for discrimination prevention.
\newblock In {\em Advances in Neural Information Processing Systems}, pages
  3992--4001, 2017.

\bibitem{celis2018classification}
L~Elisa Celis, Lingxiao Huang, Vijay Keswani, and Nisheeth~K Vishnoi.
\newblock Classification with fairness constraints: A meta-algorithm with
  provable guarantees.
\newblock {\em arXiv preprint arXiv:1806.06055}, 2018.

\bibitem{cohen2018acceleration}
Michael~B Cohen, Jelena Diakonikolas, and Lorenzo Orecchia.
\newblock On acceleration with noise-corrupted gradients.
\newblock {\em arXiv preprint arXiv:1805.12591}, 2018.

\bibitem{corbett2017algorithmic}
Sam Corbett-Davies, Emma Pierson, Avi Feller, Sharad Goel, and Aziz Huq.
\newblock Algorithmic decision making and the cost of fairness.
\newblock In {\em Proceedings of the 23rd ACM SIGKDD International Conference
  on Knowledge Discovery and Data Mining}, pages 797--806. ACM, 2017.

\bibitem{d2008smooth}
Alexandre d'Aspremont.
\newblock Smooth optimization with approximate gradient.
\newblock {\em SIAM Journal on Optimization}, 19(3):1171--1183, 2008.

\bibitem{datta2015automated}
Amit Datta, Michael~Carl Tschantz, and Anupam Datta.
\newblock Automated experiments on ad privacy settings.
\newblock {\em Proceedings on Privacy Enhancing Technologies}, 2015(1):92--112,
  2015.

\bibitem{devolder2014first}
Olivier Devolder, Fran{\c{c}}ois Glineur, and Yurii Nesterov.
\newblock First-order methods of smooth convex optimization with inexact
  oracle.
\newblock {\em Mathematical Programming}, 146(1-2):37--75, 2014.

\bibitem{Dua:2017}
Dua Dheeru and Efi Karra~Taniskidou.
\newblock {UCI} machine learning repository.
\newblock \url{http://archive.ics.uci.edu/ml}, 2017.

\bibitem{dwork2012fairness}
Cynthia Dwork, Moritz Hardt, Toniann Pitassi, Omer Reingold, and Richard Zemel.
\newblock Fairness through awareness.
\newblock In {\em Proceedings of the 3rd innovations in theoretical computer
  science conference}, pages 214--226. ACM, 2012.

\bibitem{feldman2015certifying}
Michael Feldman, Sorelle~A Friedler, John Moeller, Carlos Scheidegger, and
  Suresh Venkatasubramanian.
\newblock Certifying and removing disparate impact.
\newblock In {\em Proceedings of the 21th ACM SIGKDD International Conference
  on Knowledge Discovery and Data Mining}, pages 259--268. ACM, 2015.

\bibitem{flores2016false}
Anthony~W Flores, Kristin Bechtel, and Christopher~T Lowenkamp.
\newblock False positives, false negatives, and false analyses: A rejoinder to
  machine bias: There's software used across the country to predict future
  criminals. and it's biased against blacks.
\newblock {\em Fed. Probation}, 80:38, 2016.

\bibitem{freund1996game}
Yoav Freund and Robert~E Schapire.
\newblock Game theory, on-line prediction and boosting.
\newblock In {\em Proceedings of the ninth annual conference on Computational
  learning theory}, pages 325--332. ACM, 1996.

\bibitem{goh2016satisfying}
Gabriel Goh, Andrew Cotter, Maya Gupta, and Michael~P Friedlander.
\newblock Satisfying real-world goals with dataset constraints.
\newblock In {\em Advances in Neural Information Processing Systems}, pages
  2415--2423, 2016.

\bibitem{goodfellow2014generative}
Ian Goodfellow, Jean Pouget-Abadie, Mehdi Mirza, Bing Xu, David Warde-Farley,
  Sherjil Ozair, Aaron Courville, and Yoshua Bengio.
\newblock Generative adversarial nets.
\newblock In {\em Advances in neural information processing systems}, pages
  2672--2680, 2014.

\bibitem{hardt2016equality}
Moritz Hardt, Eric Price, Nati Srebro, et~al.
\newblock Equality of opportunity in supervised learning.
\newblock In {\em Advances in neural information processing systems}, pages
  3315--3323, 2016.

\bibitem{kamiran2012data}
Faisal Kamiran and Toon Calders.
\newblock Data preprocessing techniques for classification without
  discrimination.
\newblock {\em Knowledge and Information Systems}, 33(1):1--33, 2012.

\bibitem{kay2015unequal}
Matthew Kay, Cynthia Matuszek, and Sean~A Munson.
\newblock Unequal representation and gender stereotypes in image search results
  for occupations.
\newblock In {\em Proceedings of the 33rd Annual ACM Conference on Human
  Factors in Computing Systems}, pages 3819--3828. ACM, 2015.

\bibitem{kivinen1997exponentiated}
Jyrki Kivinen and Manfred~K Warmuth.
\newblock Exponentiated gradient versus gradient descent for linear predictors.
\newblock {\em Information and Computation}, 132(1):1--63, 1997.

\bibitem{krasanakis2018adaptive}
Emmanouil Krasanakis, Eleftherios Spyromitros-Xioufis, Symeon Papadopoulos, and
  Yiannis Kompatsiaris.
\newblock Adaptive sensitive reweighting to mitigate bias in fairness-aware
  classification.
\newblock In {\em Proceedings of the 2018 World Wide Web Conference on World
  Wide Web}, pages 853--862. International World Wide Web Conferences Steering
  Committee, 2018.

\bibitem{madras2018learning}
David Madras, Elliot Creager, Toniann Pitassi, and Richard Zemel.
\newblock Learning adversarially fair and transferable representations.
\newblock {\em arXiv preprint arXiv:1802.06309}, 2018.

\bibitem{menon2018cost}
Aditya~Krishna Menon and Robert~C Williamson.
\newblock The cost of fairness in binary classification.
\newblock In {\em Conference on Fairness, Accountability and Transparency},
  pages 107--118, 2018.

\bibitem{pleiss2017fairness}
Geoff Pleiss, Manish Raghavan, Felix Wu, Jon Kleinberg, and Kilian~Q
  Weinberger.
\newblock On fairness and calibration.
\newblock In {\em Advances in Neural Information Processing Systems}, pages
  5680--5689, 2017.

\bibitem{quadrianto2017recycling}
Novi Quadrianto and Viktoriia Sharmanska.
\newblock Recycling privileged learning and distribution matching for fairness.
\newblock In {\em Advances in Neural Information Processing Systems}, pages
  677--688, 2017.

\bibitem{PearsonCorrelation}
Inc. StatSoft.
\newblock Electronic statistics textbook.
\newblock
  \url{http://www.statsoft.com/textbook/glosp.html#Pearson%20Correlation},
  2013.

\bibitem{wadsworth2018achieving}
Christina Wadsworth, Francesca Vera, and Chris Piech.
\newblock Achieving fairness through adversarial learning: an application to
  recidivism prediction.
\newblock {\em arXiv preprint arXiv:1807.00199}, 2018.

\bibitem{woodworth2017learning}
Blake Woodworth, Suriya Gunasekar, Mesrob~I Ohannessian, and Nathan Srebro.
\newblock Learning non-discriminatory predictors.
\newblock {\em arXiv preprint arXiv:1702.06081}, 2017.

\bibitem{xu2018fairgan}
Depeng Xu, Shuhan Yuan, Lu~Zhang, and Xintao Wu.
\newblock Fairgan: Fairness-aware generative adversarial networks.
\newblock {\em arXiv preprint arXiv:1805.11202}, 2018.

\bibitem{zafar2017fairnessmis}
Muhammad~Bilal Zafar, Isabel Valera, Manuel Gomez~Rodriguez, and Krishna~P
  Gummadi.
\newblock Fairness beyond disparate treatment \& disparate impact: Learning
  classification without disparate mistreatment.
\newblock In {\em Proceedings of the 26th International Conference on World
  Wide Web}, pages 1171--1180. International World Wide Web Conferences
  Steering Committee, 2017.

\bibitem{zafar2017fairness}
Muhammad~Bilal Zafar, Isabel Valera, Manuel Gomez~Rodriguez, and Krishna~P
  Gummadi.
\newblock Fairness constraints: Mechanisms for fair classification.
\newblock {\em arXiv preprint arXiv:1507.05259}, 2017.

\bibitem{zemel2013learning}
Rich Zemel, Yu~Wu, Kevin Swersky, Toni Pitassi, and Cynthia Dwork.
\newblock Learning fair representations.
\newblock In {\em International Conference on Machine Learning}, pages
  325--333, 2013.

\bibitem{zhang2018mitigating}
Brian~Hu Zhang, Blake Lemoine, and Margaret Mitchell.
\newblock Mitigating unwanted biases with adversarial learning.
\newblock {\em arXiv preprint arXiv:1801.07593 \&}, 2018.

\end{thebibliography}

\end{document}